\documentclass{article}
\usepackage{amsmath}
\usepackage{hyperref}
\usepackage{amsfonts}
\usepackage{amssymb}
\usepackage{graphicx} 
\usepackage{hyperref}
\usepackage[normalem]{ulem}
\usepackage{authblk} 
\usepackage{tikz}

\newcommand{\be}{\begin{equation}}
\newcommand{\ee}{\end{equation}}
\title{Calculating Mutual Information between a Reward Maximizer and its Environment}

\author[$\dagger$,1]{Alfred Harwood}
\author[$\dagger$,1,2]{Jose Faustino}
\author[1]{Alex Altair}
\affil[$\dagger$]{Equal contribution}
\affil[1]{Dovetail Research}
\affil[2]{University of Sao Paulo}

\date{\today{}}

\usepackage{amsthm}

\theoremstyle{plain} 
\newtheorem{theorem}{Theorem}[section] 
\newtheorem{lemma}[theorem]{Lemma}     

\theoremstyle{definition}
\newtheorem{definition}{Definition}[section]

\theoremstyle{remark}
\newtheorem*{remark}{Remark} 

\begin{document}

\maketitle
\begin{abstract}
    An important question in the field of AI is the extent to which successful behaviour requires an internal representation of the world. In this work, we quantify the amount of information an optimal policy provides about the underlying environment. We consider a Controlled Markov Process (CMP) with $n$ states and $m$ actions, assuming a uniform prior over the space of possible transition dynamics. We prove that observing a deterministic policy that is optimal for any non-constant reward function then conveys exactly $n \log m$ bits of information about the environment. Specifically, we show that the mutual information between the environment and the optimal policy is $n \log m$ bits. This bound holds across a broad class of objectives, including finite-horizon, infinite-horizon discounted, and time-averaged reward maximization. These findings provide a precise information-theoretic lower bound on the ``implicit world model'' necessary for optimality.
    
    \vspace{1em} 
    \noindent\textbf{Keywords:} Reinforcement Learning, Information Theory, Implicit World Models, AI Safety.
\end{abstract}

\pagebreak

\section{Introduction}
\subsection{Does achieving goals require a world model?}
In the field of AI, it is common to refer to an AI agent as having a ``world model"\cite{ding2025understanding}. By this is normally meant that the agent internally contains a model of the world that the agent is interacting with, that it can use to make predictions about the world. Such models are generally only partial or imperfect, but still a good enough approximation of the world to be useful.\\

If an agent has access to a perfect world model, then (tractability aside) it can achieve optimal performance by simply querying the world model for the outcome of each possible action, and picking the one with the best outcome. But we can also ask if the relationship goes both ways. Is there a sense in which a good performance \emph{requires} a world model? This question is sometimes framed as one component of the `agent-like structure problem' \cite{altair2024towards}.\\

One motivation for asking this question is that the behaviour of machine learning systems is notoriously fragile and hard to predict. ``Out-of-distribution" generalization is a major outstanding issue. We hope that developing formal results that let us deduce world model accuracy from only observations of performance might help us to understand when we could or couldn't trust our system's performance.\\

While it is clear that some tasks do not require a full model of the world, it also seems clear that a wide variety of goals require an agent to contain some information about its world. In this paper, we operationalize this question by asking: `if we are given an optimal policy for a particular goal, how much information about the world can be extracted from it?'. \\

This framing sidesteps questions of the structure and architecture of the agent enacting the policy and focuses purely on the information-theoretic content of the policy. As such, our result makes no claims about the form or implementation that an internal world model must take, and instead focuses on how much information it must contain. \\

Following \cite{richens2025general}, we investigate this question in a setting where the environment is a controlled Markov processes. In our result the `goal' which an agent must achieve is captured by a reward function over states which is summed over time. The degree to which this framing captures a fully general notion of `goals' is subject to some debate. For example, \cite{bowling2023settling} discusses the circumstances under which this `reward hypothesis' applies, comparing and relating it to the Von Neumann Morgenstern utility theorem \cite{von2007theory}. Nonetheless, the summing of rewards over time is a very common way of framing goals in reinforcement learning,\cite{sutton1998reinforcement} and it is the approach we will use. \\

Starting with maximal ignorance about the environment means not knowing any of the individual transition probabilities. We represent this state of ignorance using a uniform probability density function over all possible environments. We then consider how our subjective probability distribution over possible environments changes when we learn that a particular policy is optimal for maximizing a specific value function in the environment. Any particular policy will only be optimal in a small subset of possible environments, so that learning this policy gives us information about the environment. In this sense, information about the environment is contained within the optimal policy. \\

We quantify this using the mutual information between the environment and the policy of an optimal reward maximiser. In particular, we show that optimal reward-maximising policies contain $n \log m$ bits of information about the world, where $n$ is the number of environment states, and $m$ is the number of actions that the policy is able to take in each state. This holds for any non-constant reward functions and holds for all common reward aggregation methods including time-discounted, finite-time, and time-averaged. 
\subsection{Related work}
In former decades, techniques of ``Good Old-Fashioned AI" favoured inputting explicit representations of the world, as chosen and designed by the human engineers. In \cite{brooks1991intelligence}, Brooks influentially argued that such representations were ineffective and unnecessary, claiming that ``the world is its own best model''. In `The Bitter Lesson' \cite{sutton2019bitter}, Sutton argued that techniques which leveraged compute for learning (and search) reliably beat engineers' attempts to deliberately encode their knowledge of the domain into ML system. Both of these however leave open the possibility that successful systems are relying on \emph{implicit} world models. At the same time, \cite{ha2018world} was influential in showing several examples where learning separate world models was effective for training good policies. \\


An early theorem that moves toward addressing our question outside of the context of AI appeared in \cite{touchette2004information}. In this paper, the authors derive a bound on the entropy reduction achievable by a controller in terms of its mutual information with the environment. While this result is illuminating, the model it considers is simplified since it only considers goals characterised by entropy reduction over a single timestep. In our work, we consider reward maximisation over multiple timesteps for a much more general framework for quantifying goal-achieving behaviour. \\

Recently, the Causal Incentives Working Group published two papers investigating versions of this problem \cite{richens2024robust, richens2025general}. The paper \cite{richens2025general} shows that, if a policy is general enough to achieve a wide range of goals, then one can extract from the policy's actions an epsilon-approximate description of the world with which it is interacting. That is to say, if you have access to the policy (and the knowledge that it is well-performing) then you can extract a world model. In this sense, the policy contains an implicit world model. Various extensions to this result, including stochastic policies and partially observable environments are investigated in \cite{santiago} and \cite{nayebi2026capable}. A similar result was shown in the setting of causal Bayesian networks in \cite{richens2024robust}.\\

There are some important assumptions made in \cite{richens2025general} which we wish to develop and relax. First, the policy is assumed to be general and `goal-conditioned' meaning that it can be fed multiple goals and achieve good performance on all of them. In particular, the agent must be general enough to succeed at a particular class of goals which the authors show is equivalent to modelling the individual transition probabilities of the environment. While this is valid as an assumption, the class of goals feels unnatural and must be expressed using Linear Temporal Logic. Here, we consider the more natural class of goals represented by reward maximization, which is standard in reinforcement learning. Furthermore, we only require that a policy is optimal for maximizing a single reward function, rather than requiring our policy to achieve general good performance on a wide range of goals. Since such policies do not contain the full information about the environment, we are able to discuss and quantify the amount of information about the world contained in the policy. Thus, our result finds natural application to finite agents with partial world models. \\

\section{Setup} \label{setup-section}
\subsection{Environments} \label{env-section}
Following \cite{richens2025general}, we consider an environment modelled by a Controlled Markov Process (CMP) with a set of $n$ states $S =\{s_1,s_2,\dots s_n\}$ and a set of $m$ actions $A =\{a_1,a_2,\dots a_m\}$ which can be taken in each state. The CMP environment is characterized by a set of transition probabilities of the form $P_{s_i s_j}(a_k)$ which specify the probability that taking an action $a_k\in A$ while in state $s_i\in S$ results in a transition to state $s_j\in S$. Time is taken to be discrete. At each timestep, a policy acting in the environment can take a single action, causing the system to evolve stochastically to its subsequent state. \\

Since taking an action in a state must result in a transition to \emph{some} state, we have the normalization conditions that $\sum_j P_{s_i s_j}(a_k) =1 \,\, \forall i,k$. Since there are $m$ possible actions and $n$ possible initial states, the environment lives in a space characterized by the product of $nm$ simplexes each with dimension $(n-1)$, making it a closed subset of $\mathbb{R}^{nm(n-1)}$. We use $\mathcal{X}$ to denote this space and use $x$ to denote a point in this space, corresponding to a full list of transition probabilities which fully characterizes a particular environment. \\
\begin{figure}
    \centering
    \includegraphics[width=0.97\linewidth]{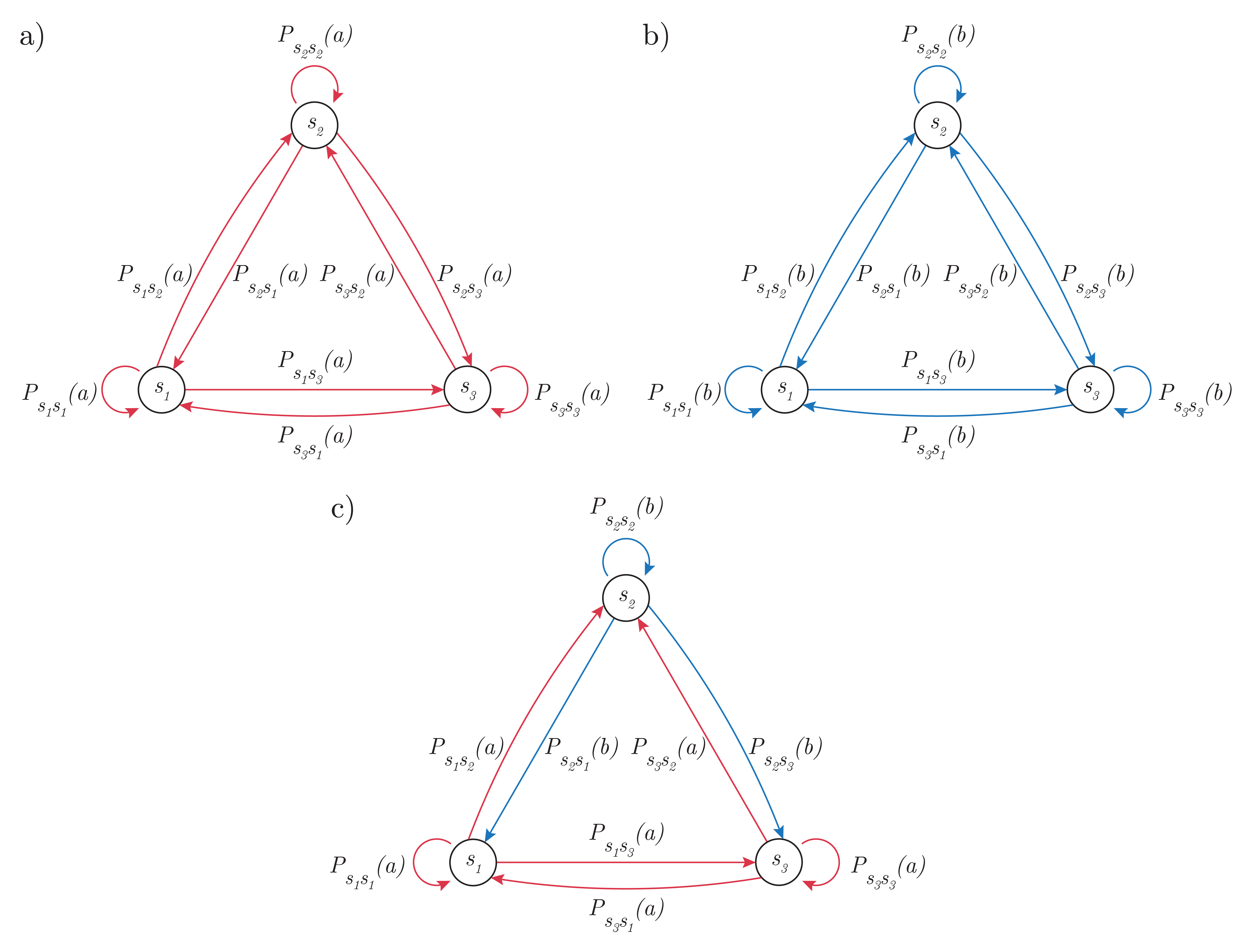}
    \caption{An example of the Markov chains resulting from three different policies in a three-state, two-action ($n=3,m=2$) Controlled Markov Process. Figure a) shows the transition probabilities when the policy takes action a in all states (ie. $\pi(s_i) = a, \forall i$). Figure b) shows the transition probabilities when the policy takes action b in all states (ie. $\pi(s_i) = b, \forall i$). Figure c) shows the transition probabilities between states for the policy which takes action $a$ in state $s_1$, action $b$ in state $s_2$ and action $a$ in state $s_3$.}
    \label{fig:CMPs}
\end{figure}

\subsection{Policies} \label{pol-section}
Acting in this environment we consider an agent using a deterministic, Markovian policy. We assume that the environment is fully observable, meaning that the agent can observe the exact state that it is in. The actions taken by such a policy are therefore entirely determined by the current environmental state meaning that it can be characterized by a deterministic function $\pi: S \rightarrow A$ which takes the current state as an input and outputs an action. Note that there are $m^n$ possible policies of this form. We will use an indexing over these policies and label the $i$-th policy as $\pi_i$ where $i$ can run from from $1$ to $m^n$. We will use the notation $\mathcal{I}=\{1,2,...m^n\}$ to denote the set of indices.\\

Notice that a deterministic, Markovian policy $\pi$ acting in an environment $x$ causes the system to behave as a stationary Markov process characterised by a transition matrix $M_{\pi}(x)$ which depends both on the policy $\pi$ and the environment $x$. The elements of the transition matrix $[M_{\pi}(x)]_{ij}$ correspond to the probability that the action taken by policy $\pi$ in environment $x$ leads to state $s_j$ transitioning to state $s_i$. In the notation from before, we can write $[M_{\pi}(x)]_{ij} = P_{s_j s_i}(\pi(s_j))$. An example CMP is shown in Figure \ref{fig:CMPs}.
\subsection{Rewards}
In our setup, goals are characterised through the maximisation of the expected reward signal over states. This reward is captured using a function $r:S \to (0,1)$. The only restriction we place on the reward function is that it is \emph{non-constant}, meaning that there is at least one pair of states $s_i$ and $s_j$ such that $r(s_i)>r(s_j)$. We use the terminology `value function' to refer to any way of accumulating this reward over multiple timesteps. There are multiple ways of doing this. We consider time-discounted summation of expected reward over both finite and infinite time-horizons and time-averaged summation of expected reward as represented in equations (\ref{eq:discountvalue}), (\ref{eq:finitevalue}) and, (\ref{eq:timeaverage}) respectively. We denote the value function evaluated in environment $x$ with policy $\pi$ as $V_{\pi}(x)$ and let the context denote the type of reward accumulation used.\\

We say that a policy $\pi_o$ is optimal for a value function in an environment $x$ if $V_{\pi_o}(x) \geq V_{\pi_i}(x) \forall i \in \mathcal{I}$ and we say that $\pi_o$ is \emph{uniquely} optimal if $V_{\pi_o}(x) > V_{\pi_i}(x) \forall i \in \mathcal{I}$. Sometimes we will refer to the evaluation of $V_\pi(x)$ as the `performance' of policy $\pi$ in environment $x$.
\subsection{Mutual Information}
We treat the environment as a random variable $X$ which takes values from the set $\mathcal{X}$. Maximum ignorance of the environment is characterized by a uniform probability density function over $\mathcal{X}$ given by $p(x) = \frac{1}{V}$ where $V = \int_{x \in \mathcal{X}} d x$ is the volume of the space $\mathcal{X}$. \\ 

The differential entropy $h$ will be used to quantify the amount of information we have about the environment $X$:
\be
    h(X) = -\int_{x \in \mathcal{X}} p(x) \log p(x) d x \, .
\ee
Initially, when $p(x)$ is a uniform probability density function (PDF) the differential entropy of $X$ is maximal and we have
\be
    h(X) = -\int_{x \in \mathcal{X}} p(x) \log p(x) d x =  -\int_{x \in \mathcal{X}} \frac{1}{V} \log \frac{1}{V} d x =  \log V \, .
\ee
In the main body of the paper, we will assume that our prior over environments $p(x)$ is uniform, however our results also apply to any prior which is invariant up to a re-labelling of actions (see Appendix \ref{ap:generalprior} for more details).
If $\Pi$ is the random variable corresponding to the optimal policy in a particular environment then knowing the optimal policy reduces the differential entropy of the environment to $h(X|\Pi)$. The expected difference between these quantities is the mutual information:
\be
    I(X;\Pi) = h(X) - \sum_{i=1}^{m^n} h(X|\Pi=\pi_i) P(\Pi = \pi_i)\, . \label{eq:mutualinformation}
\ee
Notice that, while differential entropy can depend on the choice of coordinates \cite{cover1999elements}, the mutual information between $X$ and $\Pi$ is an invariant quantity, since it can be expressed as a KL-divergence. Since all our results are in terms of mutual information, issues surrounding the coordinate-dependence of differential entropy are avoided. \\

In Section \ref{sec:average} on time-averaged rewards, we will exclude $\partial \mathcal{X}$ the boundary points of $\mathcal{X}$ and just consider the interior $\text{int(}{\mathcal{X})}$. This boundary of $\mathcal{X}$ includes all points where at least one transition probability is exactly zero. Excluding these environments allows us to assume that the environment is irreducible and aperiodic. But, since $\partial \mathcal{X}$ has zero Lebesgue measure we can exclude it without affecting any of integrals over $\mathcal{X}$ used when calculating the differential entropy (see Figure \ref{fig:Xspace}).

\section{Results} 

A state of complete ignorance about the environment variable $X$ is represented by a uniform PDF over $\mathcal{X}$ with maximal differential entropy, equal to $\log{V}$. If we obtain some information about the environment, then our PDF will change, leading to a lower differential entropy. The amount of information we have obtained about the environment will be quantified by the expected change in differential entropy which is the mutual information $I(X;\Pi)$. In this paper, we ask: how much information do we gain about the environment if we learn that a particular policy is optimal for some task within that environment?\\

Suppose that we had a goal such that in every environment, except for a measure zero subset of $\mathcal{X}$, there was exactly one optimal policy. Furthermore, assume that the volume of environments for which a particular policy $\pi_i$ was optimal was equal for all the $m^n$ different policies $\pi_i$. This would mean that we could partition $\mathcal{X}$ (the space of possible environments $\mathcal{X}$) into $m^n$ equal volume subsets depending on which policy was optimal in each subset (see Figure \ref{fig:Xspace}). In other words, if $\mathcal{X}_i \subset \mathcal{X}$ denotes the set of environments where a policy $\pi_i$ is optimal then $\bigcup_{i}\mathcal{X}_i = \mathcal{X}$ and $\int_{x \in \mathcal{X}_i} d x = \frac{V}{m^{n}}$ for all $i$. In this situation, we have the following lemma:\\

\begin{lemma}\label{lem:nlogm}
    Consider a uniform prior over environments $p(x)$. Let $X$ be the random variable representing the environment and governed by the PDF $p(x)$. \\

    Assume a goal represented by the maximization of a value function $V_\pi(x)$. We say that a policy $\pi_o$ is optimal for this value function in environment $x$ if $V_{\pi_o}(x)\geq V_{\pi_i}(x), \forall i$. Let \[\mathcal{X}_i = \{x \in \mathcal{X} \mid \pi_i \text{ is optimal in } x\}\, .\] 
    If 
    \begin{enumerate}
        \item $(\mathcal{X}_i)_{i \in \{1,...,m^n\}}$ is a partition for almost every environment - i.e, the set $ \mathcal{X}_i \cap \mathcal{X}_j$, has measure zero for all $i,j$. In other words, there is exactly one optimal policy for almost all environments
        \item $\mathcal{X}_i$ and $\mathcal{X}_j$ have the same volume, for any $i,j \in I$
    \end{enumerate}
    Then, 
    \begin{enumerate}
    \setcounter{enumi}{2}
        \item  $h(X)-h(X\mid \Pi=\pi_i) = n\log m \,,\,\forall i$, i.e, knowing which policy is optimal reduces the differential entropy of the environment by $n\log m$ bits. The mutual information between the environment and the optimal policy is therefore 
    \end{enumerate}
        \be I(X:\Pi) = n \log m \, .
        \ee
        
\end{lemma}

\begin{proof}
    Let the measure zero subset where there is more than one optimal policy be labelled $\mathcal{Z}\subset \mathcal{X}$. Removing this subset does not affect the volume so we have
    \be
        \int_\mathcal{X}dx = \int_{\mathcal{X}\ /\mathcal{Z}} dx = V \, .
    \ee
    The fact that each policy is optimal in an equal volume subset of $\mathcal{X}$ means that each policy is optimal for a volume $\frac{V}{m^n}$, which comes from equally dividing up the volume of $\mathcal{X}/\mathcal{Z}$. In this situation, if we learned which policy was optimal in an environment, this would tell us that the environment lay somewhere within a a volume $\frac{V}{m^n}$ of possible environments (a more formal explanation of this process using Bayes Theorem can be found in Appendix \ref{formalappendix}). This state of knowledge would be represented by PDF with value zero outside of the set $\mathcal{X}_i$ and a uniform distribution over the points inside $\mathcal{X}_i$:

\be
    p(x| \pi_i) = 
    \begin{cases}
        \frac{m^n}{V} \text{ when } x \in \mathcal{X}_i \\
        0 \text{ otherwise} 
    \end{cases} \, .
\ee

This PDF would have a differential entropy given by
\be
    h(x|\pi_i) =  -\int_{x \in \mathcal{X}_i} \frac{m^n}{V} \log \frac{m^n}{V}d x  = \log \frac{V}{m^n} \, .
\ee
The change in differential entropy which occurs upon learning the optimal policy is, in this situation:
\be
    h(x) - h(x|\pi_i) = \log V -\log\frac{V}{m^n} = n \log m \, .
\ee
\end{proof}

Thus, in this situation, learning the optimal policy gives us $n\log m$ bits of information about the environment. Since each policy results in the same reduction in differential entropy, we can trivially express this as mutual information between the environment and optimal policy if both are represented as random variables. Since $h(X|\pi_i)$ is the same for each $\pi_i$, the mutual information between $X$ and $\Pi$ is obtained from equation (\ref{eq:mutualinformation}):\\
\be
    I(X;\Pi) = h(X) - \sum_{i=1}^{m^n} h(X|\pi_i) P(\Pi = \pi_i) = n \log m \, .
\ee
\begin{figure}
    \centering
    \includegraphics[width=0.95\linewidth]{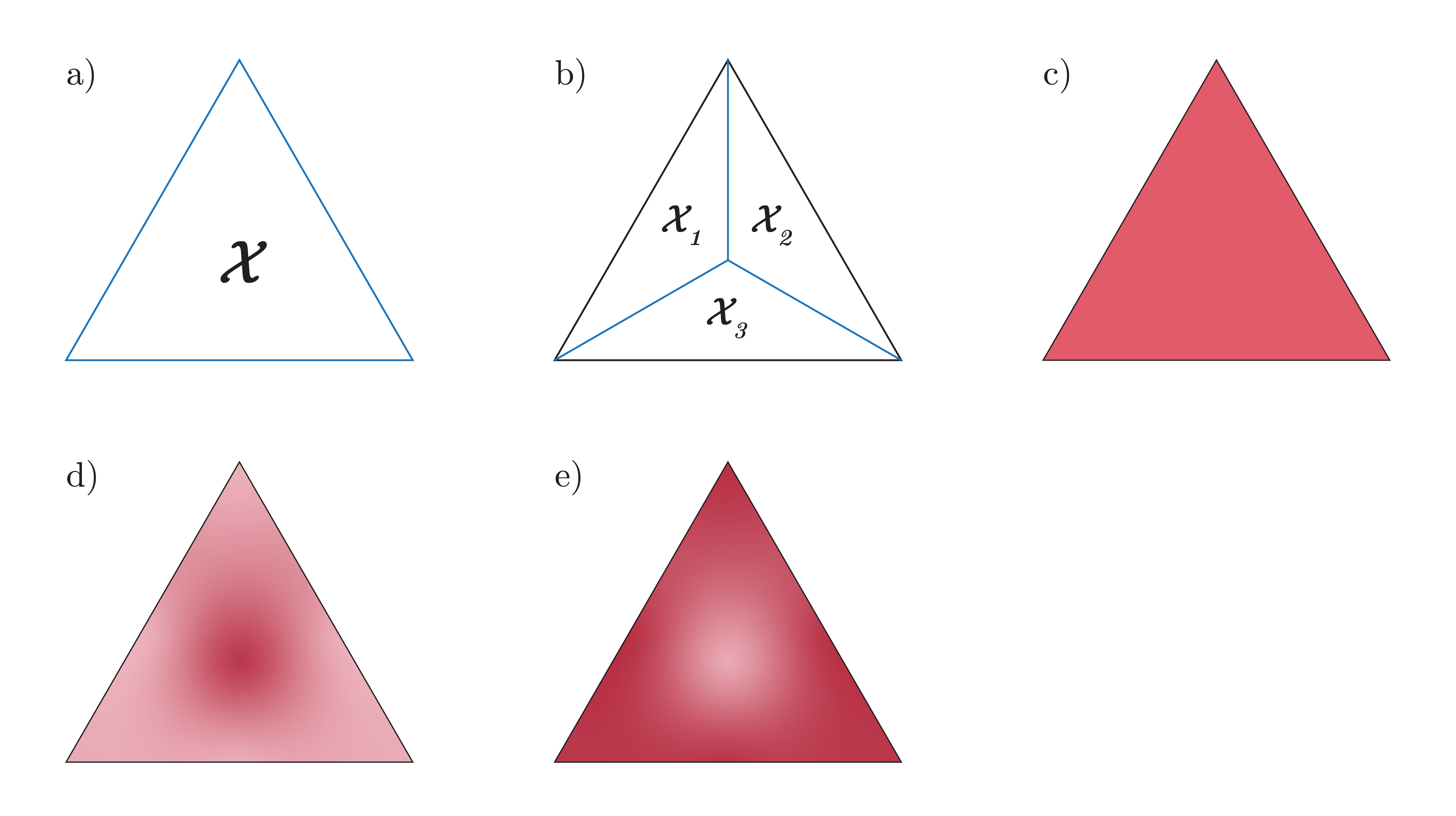}
    \caption{We consider a space of possible environments $\mathcal{X}$, here stylized as a triangle in Figure \ref{fig:Xspace}a). The boundary, highlighted in blue, is a measure zero subset which contains environments where at least one of the transition probabilities is zero. We can partition this space into equal volume subsets here labelled $\mathcal{X}_1, \mathcal{X}_2, \mathcal{X}_3$ (Figure \ref{fig:Xspace}b). The set of boundary points between these sets (again highlighted in blue) have measure zero. Starting from a uniform prior over possible environments (Figure \ref{fig:Xspace}c) learning in which of these partitions the true environment lies gives a number of bits of information equal to the logarithm of the number of cells in the partition (in this case $\log{3}$). In the main paper, we consider a uniform prior but in Appendix \ref{ap:generalprior} we extend the result to cover non-uniform prior that are symmetric with respect to the partition such as those shown in Figure \ref{fig:Xspace}d) and Figure \ref{fig:Xspace}e).}
    \label{fig:Xspace}
\end{figure}
Recall that in order to get this result, we had to make two important assumptions relating to the way in which we characterised `optimal' policies, relative to some goal. First, we assumed that there is exactly one optimal policy in each environment except a measure zero subset. We will call this `Assumption 1'. Second, we had to assume that each of the $m^n$ deterministic polices was optimal in an equal volume subset of $\mathcal{X}$. We will call this `Assumption 2'.\\

As it turns out, Assumption 2 is true for a very large class of goals. In particular, it is true for any goal involving the maximisation of a value function of the form $V_\pi(x)= k(M_{\pi}(x))$ where $M_\pi$ is the transition matrix of the Markov process over states which results from enacting a policy $\pi$ in environment $x$ and $k$ is a function which takes in a transition matrix and outputs real number. This means that the only way in which the policy influences the outcome of the value function is through its effect on the transition matrix. In this case, the following theorem applies.
\begin{theorem} \label{thm:equalvolumes}
    Let $V_\pi(x)$ be a value function which is only influenced by the policy through its effect on the transition matrix $M_\pi(x)$. Let $\mathcal{X}_i$ be the subset of environments for which policy $\pi_i$ is the unique optimal policy for maximising this value function. Then:
    \be
        \int_{\mathcal{X}_i} dx = \int_{\mathcal{X}_j} dx \quad \forall i,j
    \ee
    ie. all policies are uniquely optimal in an equal volume of environments.
\end{theorem}
\begin{proof} 
    Here, we will sketch the proof and leave the details for the Appendix \ref{ap:equalvolumes}. Consider two different policies $\pi_i$ and $\pi_j$ and let us label the actions taken by them as $\pi_i(s_k) = b_k$ and $\pi_j(s_k) = c_k$. In an environment $x$, the only transition probabilities which contribute to the performance of $\pi_i$ are those for action $b_k$, taken in state $s_k$ for $k=1\dots n$. Using notation from before, these transition probabilities are $\{P_{s_k s'}(b_k):k=1\dots n, s'\in S\}$.\\ 
    
    Now, consider the following function $g_{ij}:\mathcal{X}\rightarrow \mathcal{X}$ which takes in an environment $x$ and changes it so that $P_{s_k s'}(b_k)$ and $P_{s_ks'}(c_k)$ are swapped for all $b_k\neq c_k$ and keeps all other transition probabilities the same. Applying policy $\pi_i$ in an environment $x$ results in the same transition matrix as applying the policy $\pi_j$ in an environment $g_{ij}(x)$ so we have that $M_{\pi_i}(x) = M_{\pi_j}(g_{ij}(x))$. If $\pi_i$ was optimal in the environment $x$, then $\pi_j$ will be optimal in the new environment $g_{ij}(x)$. Thus, if $\mathcal{X}_i$ is the set of environments for which $\pi_i$ is optimal, then $g_{ij}(\mathcal{X}_i)$ a set for which $\pi_j$ is optimal. Since (as proved in the appendix) $g_{ij}$ is invertible and measure-preserving, the volumes of $\mathcal{X}_i$ and $\mathcal{X}_j$ are equal. Since this argument can be applied to $\emph{any}$ pair of policies $\pi_i$ and $\pi_j$, we can conclude that the volumes of environments for which each policy is optimal are equal.
\end{proof}

In what follows, we will show that these Assumptions 1 and 2 hold for a very general class of goals. In particular, they hold when the goal is the maximisation of any non-constant reward function over environmental states. We consider the time-averaged, discounted, and finite-time total reward maximisation and show that all of these settings satisfy these two assumptions. Therefore, optimal policies for any goal captured by these kinds of reward maximisation contain $n \log m$ bits of information about their environment. In Appendix \ref{ap:generalprior} we prove an extension to these theorems which shows that these results (that optimal policies contain $n \log m$ bits of information about their environment) also apply when the prior is not uniform but is invariant under a re-labelling of actions.
\subsection{Proof Strategy and Notation} \label{sec:proofstrategy}
For all of these reward maximisation frameworks, our proof follows roughly the same steps. First, we ensure that the value function $V_{\pi}(x)$ is only influenced by the policy through its effect on the transition matrix, meaning that Theorem \ref{thm:equalvolumes} applies. Then, we show that the the value function is a real analytic function of $x$ and define a `difference function' $f_{ij}(x) = V_{\pi_i}(x)-V_{\pi_j}(x)$ which captures the difference in performance between policies $\pi_i$ and $\pi_j$ in an environment $x$. Having done this, we make use of the following standard lemma.
\begin{lemma} \label{lem:measurezero}
    For a real analytic function $f(x)$, either the set $\{x:f(x)=0\}$ has zero measure or the function is identically zero.
\end{lemma}
\begin{proof}
    A proof can be found in \cite{mityagin2015zero}.
\end{proof}

Applying this lemma to $f_{ij}(x)$ tells us that, for any pair of policies $\pi_i$ and $\pi_j$, the set of environments for which their performance is the same has measure zero. This means that, in all environments except a set of measure zero, there is a unique optimal policy for maximising the value function $V_{\pi}(x)$. This allows us to apply Lemma \ref{lem:nlogm} to the value function, yielding the result. A schematic diagram, showing the proof structure can be found in Figure \ref{fig:proof_strategy} in Appendix \ref{ap:proof_strategy}.\\

We will now introduce some terminology and notation used when discussing reward functions throughout the paper. Let $r:S\rightarrow (0,1)$ be a reward function over states and let $S_t$ be the random variable corresponding to the environment state at time $t$. The random variable $S_t$ is distributed depending on the environment $x$ along with the previous state $S_{t-1}$ and the action taken by the policy at the previous timestep $A_t=\pi(S_{t-1})$. For a fixed environment and a deterministic, Markovian policy of the kind we are considering, we can write $P(S_t=s_i)$ as
\be
    P(S_t=s_i) = \sum_{s_j\in S}P(S_t=s_i|X=x, S_{t-1}=s_j, A_t=\pi(s_j))P(S_{t-1}=s_j) \, .
\ee 
Notice that $P(S_t=s_i|X=x, S_{t-1}=s_j, A_t=\pi(s_j))$ is simply equal to the transition matrix element $[M_{\pi}(x)]_{ji}$. Now let us introduce the notation $\vec{v}_t$ to write the probability distribution over states as a vector. The elements of $\vec{v}_t$ will correspond to the probabilities $P(S_t=s_i)$. As is standard in the study of Markov processes, this allows us to write the evolution of the probability distribution over states as the vector equation $\vec{v}_{t+1} = M_{\pi}(x)\vec{v}_{t}$. We will use this notation to express value functions in terms of the transition matrix, allowing us to apply Theorem \ref{thm:equalvolumes}.\\

As a final piece of notation, it will be convenient to introduce an `expected reward' function, which is simply the expected value of reward function $r$ when applied to a distribution over states represented by a vector $\vec{v}$. We will denote this function $R:\mathbb{R}^n \rightarrow (0,1)$ so we have:
\be 
    R(\vec{v}) := \sum_i r(s_i)v_{i} =\sum_i r(s_i)P(s_i) \,,
\ee 
where $v_i$ is the element of $\vec{v}$ corresponding to $P(s_i)$. We will now apply this framework to various forms of reward maximisation.
\subsection{Reward Maximisation with discount rate} \label{sec:discount}
First, we consider a value function which is the time-discounted expected reward for a reward function over states. These value functions take the form:
\be \label{eq:discountvalue}
    V_\pi(x) = \sum_{t=1}^\infty \gamma^t \mathbb{E}[r(S_t)] \, ,
\ee
where $0<\gamma <1$ is a discount rate. Using the notation introduced in Section \ref{sec:proofstrategy}, we can express this in terms of the transition matrix and an initial distribution over states represented by a vector $\vec{v}_0$:
\be \label{eq:discountreward}
    V_{\pi}(x) = \sum_{t=1}^{\infty} \gamma^t R(M_\pi(x)^t \vec{v}_0) \,.
\ee
Writing the value function in this way shows that it is only influenced by the policy through its effect on the transition matrix.  We will now show that $V_\pi(x)$ is a real analytic function of $x$. The argument we will present here applies for any starting distribution over states $\vec{v}_0$. \\
\begin{lemma} \label{lem:discountanalytic}
    $V_\pi(x)$ is a real analytic function over $\mathcal{X}$.
\end{lemma}
\begin{proof}
First, note that, due to the linearity of the expected value, we can write:
\be 
    V_{\pi}(x) = R\Big(\sum_{t=1}^{\infty} \gamma^t M_\pi(x)^t \vec{v}_0 \Big) \,.
\ee
Since $M_\pi(x)$ is stochastic all of its eigenvalues have magnitude $|\lambda|\leq 1$. Since $\gamma<1$, we have that the matrix $A=\gamma M_{\pi}(x)$ has all eigenvalues with $|\lambda|<1$. As a result we can apply the matrix geometric series:
\be
    \sum_{t=1}^{\infty} A^t = A\sum_{t=0}^{\infty} A^t =A(1 - A)^{-1} =A \frac{C^{\intercal}}{\det(1-A)} \, ,
\ee
where $C$ is the matrix of cofactors of $(1-A)$ and $\det(1-A)$ is the determinant. Due to the restrictions of the eigenvalues of $A$, we have that $\det{(1-A)} \neq0$ so the equation is well-defined. Since both $C$ and $\det(1-A)$ are polynomials in the elements of $A$, each element of $\sum_{t=1}^{\infty} A^t$ will be a real-analytic function of the elements of $A$. Furthermore, each element of $A$ is simply an element of $M_\pi(x)$ multiplied by $\gamma$. In turn, each element of $M_\pi(x)$ comes directly from $x$. Hence, we have that the elements of $\sum_{t=1}^{\infty} A^t$ are real analytic functions of $x$. Applying this matrix to $\vec{v}_0$ and taking the expected reward with $R$ involves summing these elements and multiplying them by real numbers (probabilities and rewards) which preserves analyticity. Therefore $V_\pi(x)$ is a real analytic function of $x$.\\
\end{proof}
Now, we will define a `difference function' $f_{ij}(x) = V_{\pi_i}(x) - V_{\pi_j}(x)$. Since $V_{\pi}(x)$ is a real analytic function of $x$, we have that $f_{ij}(x)$ is also a real analytic function of $x$. In order to apply Lemma \ref{lem:measurezero}, we will now show that $f_{ij}(x)$ is not identically zero.
\begin{lemma} \label{lem:discountnotzero}
    $f_{ij}(x)$ is not identically zero
\end{lemma}
\begin{proof}
The two policies $\pi_i$ and $\pi_j$ must take a different action in at least one state which we will call $s_a$. Let us call the state with the highest reward $s_b$ which may or may not be the same as $s_a$. Consider the following environment $x'$. In this environment, in all states except $s_a$, all actions result in the next state being a random state from a uniform distribution over states. The action taken by $\pi_i$ in $s_a$ causes the state to transition to $s_b$ with probability 1, but the action taken by $\pi_j$ in state $s_a$ leads to a uniform distribution over all states. The transition matrix $M_{\pi_j}(x)$ will therefore have all elements equal to $\frac{1}{n}$. The transition matrix $M_{\pi_i}(x)$ will have all elements equal to $\frac{1}{n}$ \emph{except for} the $a^{th}$ column. The $a^{th}$ column of $M_{\pi_i}(x)$ will be all zeroes except for the $b^{th}$ entry which will be one. \\

Notice that, for any starting probability vector $\vec{v}_0$, the distribution over states for all subsequent timesteps will be uniform when applying $\pi_j$. However, when applying $\pi_i$ to distribution $\vec{v}$, the probability that that system will transition to state $s_b$ is equal to $p(s_b) = v_a+ \sum_{i\neq a} \frac{1}{n}v_i >\frac{1}{n}$. The probability will be uniform for all other states when applying $\pi_i$. As a result, the expected reward for policy $\pi_i$ at the first timestep will be larger than the expected reward achieved by $\pi_j$. Furthermore, this will be true at each subsequent timestep. Thus, when summing discounted rewards, we can conclude that $V_{\pi_i}(x')>V_{\pi_j}(x')$ and hence $f_{ij}(x')>0$. \\
    
Our example of $x'$ above required that one of the transitions was deterministic, and hence required that $x'$ was on the boundary of $\mathcal{X}$, but we can use the continuity of $f_{ij}$ to show that there exists $x\in \text{int}(\mathcal{X})$ which also satisfies $f_{ij}(x)\neq0$. Since $f_{ij}$ is analytic it is also continuous. Hence, since $f_{ij}(x')>0$ and $f_{ij}$ is continuous on $\mathcal{X}$, it's true that there is a $\delta>0$ such that for any $x \in \mathcal{X}$, if $d(x,x')<\delta$, then $f_{ij}(x)>0$. Now, consider $B_{\delta}(x')$ to be the ball of radius $\delta$ and centred in $x'$. Since $x' \in \partial \mathcal{X}$, $B_{\delta}(x') \cap \text{int}(\mathcal{X}) \neq \emptyset$. Then, let $x \in B_{\delta}(x') \cap \text{int}(\mathcal{X})$, it is true that $x \in \text{int}(\mathcal{X})$ and $f_{ij}(x)>0$, and hence $f_{ij}(x) \neq 0$.
\end{proof}
This can be applied to prove the following.
\begin{lemma} \label{lem:discountunique}
    The set of environments with more than one optimal policy for maximising the time-discounted reward have measure zero.
\end{lemma}
\begin{proof}
    By Lemma \ref{lem:measurezero}, $f_{ij}$ must either be identically zero or the set of its zeros must be measure zero. Since, by Lemma \ref{lem:discountnotzero} it is not identically zero, the set of points for which it equals zero (ie. the set of environments for which there are two policies with identical performance) must be measure zero.
\end{proof}

Additionally, since value functions of the form given in equation (\ref{eq:discountreward}) are functions of the transition matrix $M_\pi(x)$, we can apply Theorem \ref{thm:equalvolumes} meaning that each policy is optimal in an equal volume of environments. This allows us to prove the following theorem.
\begin{theorem} \label{thm:two}
    An optimal deterministic policy for maximising the discount reward over states for any non-constant reward function and any discount rate $0<\gamma<1$ contains $n\log m$ bits of information about its environment.
\end{theorem}
\begin{proof}
    By Lemma \ref{lem:discountunique} all environments, except those of a measure zero subset, have a unique optimal policy. By Theorem \ref{thm:equalvolumes}, each policy is uniquely optimal in a subset of the same volume. Therefore, by Lemma \ref{lem:nlogm}, the optimal policy for maximising the time-discounted reward contains $n \log m$ bits of information about its environment.
\end{proof}

\subsection{Reward Maximisation over Finite time horizons} \label{sec:finite}
In the previous section, we considered summing reward in the limit of infinite timesteps. Here, we show that our argument equally applies to finite-time reward maximisation. When reward is summed over finite timesteps, we can also consider cases with no discount rate (ie, $\gamma=1$) since the accumulated reward will be finite. \\

We will consider value functions of the form:
\be \label{eq:finitevalue}
    V_\pi(x) =\sum_{t=1}^T \gamma^t \mathbb{E}[r(S_t)] \, ,
\ee
for any $0<\gamma\leq1$ where $\gamma=1$ corresponds to total reward summation with no discount rate. As before, we will re-write this in terms of the transition matrix:
\be 
    V_\pi(x) =\sum_{t=1}^T \gamma^t R(M_\pi(x)^t \vec{v}_0) \, .
\ee

Since this will be a degree $T$ polynomial in the elements of $M_\pi(x)$, we have that every value function of this form will be a real analytic function of $x$. This implies that functions of the form $f_{ij}(x) = V_{\pi_i}(x) - V_{\pi_j}(x)$ are also real analytic (note that since we have re-defined $V_\pi(x)$, we have similarly implicitly re-defined $f_{ij}(x)$ to be in terms of these finite-time value functions). To apply Lemma \ref{lem:measurezero} in order to show that there is a unique optimal policy in almost all environments we must therefore show that $f_{ij}$ is not identically zero. In order to do this, we need to make a minimal qualification that either $T>1$ or $\vec{v}_0$ has full support.
\begin{lemma} \label{lem:notzerofinite}
    $f_{ij}$ is not identically zero for value functions with $T>1$ or value functions where $\vec{v}_0$ has full support.
\end{lemma}
\begin{proof}
    We can follow the proof used in Lemma \ref{lem:discountnotzero} and construct an environment where $\pi_i$ (where it differs from $\pi_j$) steers the environment to high reward states and $\pi_j$  steers the environment to a uniform distribution. However, we need to be careful since there is one situation where this does not apply. \\
    
    Suppose $T=1$ (ie. the value function only cares about maximising reward over the first timestep). Then, let $\vec{v}_0$ be a distribution without full support. We could then find two policies which were identical for all states where $\vec{v}_0$ was nonzero and differed only in states assigned zero probability mass by $\vec{v}_0$. In this situation $f_{ij}$ would be identically zero, since both policies would only be assessed over one timestep for their performance in states for which they behaved identically. This situation can be prevented by either enforcing that $\vec{v}_0$ has full support (since then all policy actions contribute to the final reward) or ensuring that the reward is summed over all timesteps (ie. $T>1$). Since we can assume that our environment is in $\text{int}(\mathcal{X})$ without affecting the mutual information calculation, in the second timestep, we are guaranteed to have a distribution over states with full support, since all actions taken in the first timestep have some probability of leading the environment to any of the other states. Thus, in this situation, $f_{ij}$ will not be identically zero provided $T>1$ or that $\vec{v}_0$ has full support.
\end{proof}
We can now prove the following lemma.
\begin{lemma} \label{lem:finiteunique}
    For all value functions of the form in equation (\ref{eq:finitevalue}), with either $T>1$ or $\vec{v}_0$ having full support, there is a unique optimal policy in all environments, except for a measure zero subset.
\end{lemma}
\begin{proof}
    Value functions of the form given in equation (\ref{eq:finitevalue}) are real analytic functions of the environment $x$. By Lemma \ref{lem:measurezero} we have that real analytic functions of the form $f_{ij}(x) = V_{\pi_i}(x) - V_{\pi_j}(x)$ will either be identically zero or have a set of zeros which is measure zero. By Lemma \ref{lem:notzerofinite}, $f_{ij}$ is not identically zero provided that $T>1$ or $\vec{v}_0$ has full support. This means that the set of environments for which $f_{ij}(x)=0$ and two policies have the same performance is of measure zero. Therefore, there is a unique optimal policy in all environments except a measure zero subset.
\end{proof}

Now, note that since we have shown that this family of value functions depends on the policy only through $M_\pi(x)$, it is covered by Lemma \ref{thm:equalvolumes}. This means that we can conclude that all policies have an equal volume of environments for which they are optimal. This allows us to prove the following theorem.
\begin{theorem} \label{thm:three}
    An optimal deterministic policy for maximising the summed reward over states for a finite number of timesteps $T$ with any discount rate $0<\gamma\leq 1$ for any non-constant reward function contains $n\log m$ bits of information about its environment provided that $T>1$ or $\vec{v}_0$ has full support.
\end{theorem}
\begin{proof}
    By Lemma \ref{lem:finiteunique} we have that there is a unique optimal policy for this family of value functions in all environments except for a measure zero subset. By Theorem \ref{thm:equalvolumes} we have that each of these policies will be optimal in an equal volume of environments. Therefore, we can apply Lemma \ref{lem:nlogm} to this family of value functions and conclude that the optimal policy contains $n \log m$ bits of information about its environment.
\end{proof}
\subsection{Time-averaged reward maximisation} \label{sec:average}
Having considered finite-time reward maximisation and time-discounted reward maximisation, we now consider time-averaged reward maximisation of a reward function. Value functions of this kind take the following form:

\be \label{eq:timeaverage}
    V_{\pi}(x) = \lim_{T\rightarrow \infty}\frac{1}{T}\sum_{t=0}^T \mathbb{E}[r(S_t)] \, .
\ee
Again, we can re-write this in terms of the transition matrix 
\be \label{eq:timeaveragevector}
    V_{\pi}(x) = \lim_{T\rightarrow \infty}\frac{1}{T}\sum_{t=0}^T R(M_{\pi}(x)^t \vec{v}_0) \, .
\ee
Notice again that the only way in which the policy affects the value function is through the transition matrix, meaning that Theorem \ref{thm:equalvolumes} applies. Again, the argument we will present applies to value functions defined in terms of \emph{any} starting distribution $\vec{v}_0$.\\

First we will use the following lemma.
\begin{lemma} \label{lem:averagerealanalytic}
    Any value function $V_{\pi}(x)$ of the form given in equation (\ref{eq:timeaverage}) is a real analytic function of $x$ over $\text{int}(\mathcal{X})$.
\end{lemma}
\begin{proof}
    The proof is straightforward but lengthy and involves application of standard results about stationary Markov processes. It can be found in Appendix \ref{ap:averageanalytic}.
\end{proof}

This also means that $f_{ij}(x) = V_{\pi_i}(x) - V_{\pi_j}(x)$ is a real analytic function of $x$. We now repeat the same process as in previous sections, first ruling out that $f_{ij}$ is identically zero and then applying Lemma \ref{lem:measurezero}. 
\begin{lemma} \label{lem:notzero}
    $f_{ij}(x)$ is not identically zero.
\end{lemma} 
\begin{proof}
    Again, as in Lemmas \ref{lem:discountnotzero} and \ref{lem:notzerofinite} we can construct an environment where $\pi_i$ has higher average reward than $\pi_j$. The construction is slightly different in this case and is done explicitly in Appendix \ref{ap:notzero}.
\end{proof}

\begin{lemma} \label{lem:timeaverageunique}
    The set of environments with more than one optimal policy for maximizing the time-averaged reward have measure zero.
\end{lemma}
\begin{proof}
    Since by Lemma \ref{lem:averagerealanalytic}, we have that $V_\pi(x)$ is real analytic for all policies $\pi$, we also have $f_{ij}(x) = V_{\pi_i}(x) - V_{\pi_j}(x)$ is real analytic. By Lemma \ref{lem:measurezero}, the set $\{x: f_{ij}(x)=0\}$ must have measure zero, or $f_{ij}(x)$ is identically zero. By Lemma \ref{lem:notzero}, we know that $f_{ij}(x)$ is not identically zero, therefore $\{x: f_{ij}(x)=0\}$ has measure zero. Therefore, there will be a unique policy in all environments except a measure zero subset.
\end{proof}

This justifies our requirement that there is exactly one optimal policy in all environments except a measure zero subset of $\mathcal{X}$. Furthermore, equation (\ref{eq:timeaveragevector}) shows that our value function can be expressed as a function of the transition matrix, meaning that we can apply Theorem \ref{thm:equalvolumes} in order to prove our final theorem.

\begin{theorem} \label{thm:four}
    An optimal deterministic policy for maximising the time-averaged reward over states for any non-constant reward function contains $n\log m$ bits of information about its environment.
\end{theorem}
\begin{proof}
    By Lemma \ref{lem:timeaverageunique} all environments, except those of a measure zero subset, have a unique optimal policy. By Theorem \ref{thm:equalvolumes}, each policy is uniquely optimal in a subset of the same volume. Therefore, by Lemma \ref{lem:nlogm}, the optimal policy for maximising the time-averaged reward contains $n \log m$ bits of information about its environment.
\end{proof}

\section{Conclusion}
We have shown that if we observe an optimal agent, starting with zero belief about the world in which the agent is (represented by a uniform prior), the knowledge that the agent is implementing a specific deterministic policy provides exactly $n \log m$ bits of information. Mathematically, this means the mutual information between the environment and the optimal policy is $n \log m$, for a very broad class of rewards. More specifically, we've shown this is true for any reward that depends only on states, regardless of the regime being finite horizon,  infinite discounted horizon or time-averaged. \\

To show this result, we established some general facts about the space of environments. First, we proved the deterministic optimal policies for reward maximizers induce a finite partition over the environments, for almost every environment. Each set of the partition is defined by environments that are optimal for a deterministic policy, and there is one such set for every distinct deterministic policy. Since there are $m^n$ possible deterministic policies, the induced partition contains $m^n$ sets. We highlight the fact that this result is strictly geometric and does not depend on any specific prior over environments. Thus, knowing the optimal policy lets us conclude that the environment is in a certain set of the partition. We then used this fact to show that knowledge of the optimal policy gives $\log m^n$ bits of information about the environment when one uses it to update from state of maximum ignorance.\\

Importantly, these results hold for a wide class of reward accumulation methods and are not an artifact of a specific choice of reward function. More generally, as long as the value function depends only on the transition probabilities of the environment, the result still holds. \\

One of the limitations of the present study is that we considered the policy $\pi$ to be deterministic and without memory. Future work would involve extending this result to non-deterministic policies of the form $\pi : S \to \Delta A$, where $\Delta A$ is the set of probability distributions over actions. This would make the result stronger since often real world agents are non-deterministic. Additionally, we would like to consider agents with memory, whose policies take in histories as well as the current state. Another extension of this result could consider partially observable environments, i.e, a setup such that there is a (possibly stochastic) function $f : S \to O$ and policies $\pi : O \to A$. In this setup, for any state $s$, $f(s) \in O$ represents the `observation' received by the agent, which may not convey full information about the environment state. \\

In this paper, we considered only optimal policies, but an obvious extension is to investigate world information contained by non-optimal policies. If a policy was not optimal but achieved a performance within $\delta$ of the optimal performance, could we bound the mutual information between the policy and the environment in terms of $\delta$? Another natural extension is to consider a wider class of goals, such as those involving reward functions which depend on actions as well as states. Finally, we are also interested in proving a similar result using the notion of mutual information from Algorithmic Information Theory instead of Shannon information.  \\

\section{Acknowledgments}
This work was funded by the Advanced Research + Invention Agency (ARIA) Safeguarded AI Programme through project code MSAI-SE01-P005. Thanks to Santiago Cifuentes, Guillermo Martin, Winter Cross for reading and commenting on the draft of this manuscript. Figures \ref{fig:CMPs} and \ref{fig:Xspace} were designed by Melisa Morales from Science Graphic Design.

\section*{Appendix}
\appendix

\setcounter{equation}{0}
\section{Visual Diagram} \label{ap:proof_strategy}
\begin{figure} [!t]
\centering
\resizebox{\textwidth}{!}{%
\begin{tikzpicture}[
    >=stealth,
    box/.style = {rectangle, draw=black, thick, rounded corners=4pt, align=center, inner sep=8pt, font=\sffamily\small},
    goal/.style = {box, fill=green!10, text width=6cm},
    core/.style = {box, fill=blue!10, text width=10.5cm},
    theorem/.style = {box, fill=orange!10, text width=7.2cm},
    bridge/.style = {box, fill=purple!10, text width=7.2cm},
    frame/.style = {box, fill=yellow!10, text width=4.2cm}
]

\node[goal] (goal) at (0, 0) {
    \textbf{Main Theorems (\ref{thm:two}, \ref{thm:three}, \& \ref{thm:four})}\\
    $I(X;\Pi) = n \log m$
};

\node[core] (lemma1) at (0, -2.8) {
    \textbf{Lemma \ref{lem:nlogm} (Partition Bound)}\\
    (1) optimal policy is unique almost everywhere and\\
    (2) partition volumes are equal $\left(\int_{\mathcal{X}_i} dx = \int_{\mathcal{X}_j} dx\right)$\\
    $\implies I(X;\Pi) = n \log m$
};

\node[theorem] (lemma2) at (-4.8, -6.8) {
    \textbf{Lemma \ref{lem:measurezero} (Measure Zero Roots)}\\
    Difference function $f_{ij}$ is real analytic and $f_{ij} \not\equiv 0$\\
    $\implies$ its roots have measure zero\\
    (Yields Unique Optimality)
};

\node[theorem] (thm1) at (4.8, -6.8) {
    \textbf{Theorem \ref{thm:equalvolumes} (Equal Volumes)}\\
    Value function $V_\pi(x)$ depends only on $M_\pi(x)$\\
    $\implies$ transformation $g_{ij}$ preserves volume\\
    (Yields Equal Partition Volumes)
};

\node[bridge] (analytic) at (-4.8, -10.8) {
    \textbf{Analytic \& Non-Zero Proofs}\\
    We can prove $V_\pi$ is real analytic and $f_{ij} \not\equiv 0$\\
    \vspace{4pt}
    \scriptsize{\textit{Lemmas \ref{lem:discountanalytic} \& \ref{lem:discountnotzero} (Discounted)}}\\
    \scriptsize{\textit{Lemmas \ref{lem:notzerofinite} \& \ref{lem:finiteunique} (Finite)}}\\
    \scriptsize{\textit{Lemmas \ref{lem:averagerealanalytic} \& \ref{lem:notzero} (Averaged)}}
};

\node[bridge] (markov) at (4.8, -10.8) {
    \textbf{Policy only influences the value via transition matrix}\\
    We can express the value function strictly via\\
    the state distribution and transition matrix:\\
    $V_\pi(x) = f(M_\pi(x)^t \vec{v}_0)$
};

\draw[fill=gray!5, draw=gray!40, dashed, thick, rounded corners=8pt] (-8.2, -17.3) rectangle (8.2, -12.8);

\node[font=\sffamily\bfseries, text=gray!80!black] at (0, -13.4) {Reward Maximization Regimes};

\node[frame] (discount) at (-5.8, -16.0) {
    \textbf{Time-Discounted}\\
    $V_\pi = \sum_{t=1}^\infty \gamma^t R(M_\pi^t \vec{v}_0)$
};

\node[frame] (finite) at (0, -16.0) {
    \textbf{Finite-Time}\\
    $V_\pi = \sum_{t=1}^T \gamma^t R(M_\pi^t \vec{v}_0)$
};

\node[frame] (average) at (5.8, -16.0) {
    \textbf{Time-Averaged}\\
    $V_\pi = \lim\limits_{T \to \infty} \frac{1}{T} \sum_{t=0}^T R(M_\pi^t \vec{v}_0)$
};


\draw[->, thick] (lemma1.north) -- (goal.south);

\draw[thick] (lemma2.north) -- (-4.8, -4.6) -- (-2.2, -4.6) node[midway, above, font=\scriptsize\sffamily] {Condition 1};
\draw[->, thick] (-2.2, -4.6) -- ([xshift=-2.2cm]lemma1.south);

\draw[thick] (thm1.north) -- (4.8, -4.6) -- (2.2, -4.6) node[midway, above, font=\scriptsize\sffamily] {Condition 2};
\draw[->, thick] (2.2, -4.6) -- ([xshift=2.2cm]lemma1.south);

\draw[->, thick] (analytic.north) -- (lemma2.south);
\draw[->, thick] (markov.north) -- (thm1.south);

\draw[thick] (-5.8, -14.6) -- (5.8, -14.6);

\draw[thick] (discount.north) -- (-5.8, -14.6);
\draw[thick] (finite.north) -- (0, -14.6);
\draw[thick] (average.north) -- (5.8, -14.6);

\draw[->, thick] (-4.8, -14.6) -- (analytic.south);
\draw[->, thick] (4.8, -14.6) -- (markov.south);

\end{tikzpicture}%
}
\caption{Visual diagram relating all proofs and results of this work. All three reward regimes are shown to possess the structural and analytic properties necessary to satisfy the conditions of the core Partition Bound (Lemma 1).}
\label{fig:proof_strategy}
\end{figure}
\section{Formal setup and entropy reduction} \label{formalappendix}

\begin{definition}
Let $S$ be the set of states and $A$ be the set of actions. We assume $S$ and $A$ are non-empty finite sets, and we denote $|S|=n,|A|=m$
\end{definition}

\begin{definition} \label{def-envs}
An environment $x$ is a list of transition probabilities 
\be
    (P(s'|s,a))_{s',s \in S, a \in A} \ee
    i.e, $\forall s \in S,a \in A$, we have $\sum_{s' \in S} P(s'|s,a)=1$ and $P(s'|s,a) \in [0,1]$.
We denote $P(s'|s,a)=P_{ss'}(a)$ and \be\mathcal{X} = \{(P_{ss'}(a))_{s',s,a} \in \mathbb{R}^{mn^2} ; \sum_{s' \in S} P(s'|s,a)=1, P(s'|s,a) \in [0,1], \forall s,a\}\ee 
is the space of all environments.
\end{definition}

\begin{definition}
    Consider a probability space $(\Omega, \mathcal{F}, \mathbb{P})$ and the measurable space
    $(\mathcal{X},\mathcal{B}_{\mathcal{X}})$, where $\mathcal{B}_{\mathcal{X}}$ is the Borel $\sigma$-algebra over $\mathcal{X}$. Then, given an absolutely continuous random variable $X : \Omega \to \mathcal{X}$ with density $f$, one can define the (differential) entropy $h(X)$ given by 
    \be
    h(X) = -\int_{\mathcal{X}} f(x) \log f(x) dx 
    \ee
\end{definition}

\begin{remark}
    We can see an environment as a point $x \in \mathcal{X}$ in the space of environments or as a random variable $X : \Omega \to \mathcal{X}$
\end{remark}

\begin{remark}
    Note that the volume of the space $\mathcal{X}$ is given by \be
    V = \int_{\mathcal{X}} dx
    \ee
\end{remark}

\begin{lemma}
    If $X \sim \mathcal{U(X)}$, then $h(X) = \log V$
\end{lemma}
\begin{proof}
 Since $X \sim \mathcal{U(X)}$, $f(x) = \frac{1}{V}$. Then, \be h(x) = -\int_{\mathcal{X}} f(x) \log f(x) dx  
     =-\int_{\mathcal{X}} \frac{1}{V} \log \frac{1}{V} dx 
    = \log V\ee

\end{proof}

\begin{lemma}\label{thm:ent-red}
    Consider a finite partition $(\mathcal{X}_j)_{j=1}^N$ of the space $\mathcal{X}$ such that all $\mathcal{X}_j$ have the same volume. Then, the differential entropy reduction achieved by the information that the environment is in $\mathcal{X}_j$ for some $j$ is given by $h(X) - h(X\mid X \in \mathcal{X}_j) =\log N$  
\end{lemma}
\begin{proof}
    Using Bayes Theorem, $f(x|x \in \mathcal{X}_j) =\frac{f(x)\mathbb{I}_{\mathcal{X}_j}(x)}{\mathbb{P}(x \in \mathcal{X}_j)}  
        = \frac{N}{V}\mathbb{I}_{\mathcal{X}_j}(x)$\\

Now, noting that $\int_{\mathcal{X}_j} dx = \frac{V}{N}$,  
\begin{align}
    h(x) - h(x\mid x \in \mathcal{X}_j) & = \log V - \int_{\mathcal{X}} \frac{N}{V}\mathbb{I}_{\mathcal{X}_j}(x) \log(\frac{N}{V}\mathbb{I}_{\mathcal{X}_j}(x))dx \\
    &=\log V - \log \frac{N}{V} \\
    &= \log N
\end{align}
\end{proof}

\begin{remark}
    Note that, more formally, we could define the set of environments, $\mathcal{X}$, as the cartesian product of $mn$ copies of $(n-1)-$dimensional simplexes. Then the environment $X$ can be treated as the identity random variable on the probability space $(\mathcal{X}, \mathcal{B}_{\mathcal{X}}, \mathbb{P})$, where $\mathbb{P}$ is a probability measure proportional to the product measure defined by the product of the Lebesgue measures in each of the individual simplexes.
\end{remark}

\section{The entropy reduction bound}
It is possible to show result \ref{thm:ent-red} is an upper bound for any other finite partition, i.e, the entropy reduction from knowing the environment is within some set of any finite partition comprised of measurable sets is at most $\log N$.\\

Let $X : \Omega \to \mathcal{X}$ an absolutely continuous random variable, with density $f(x)$. Let $\mathcal{X}_1,\mathcal{X}_2,...,\mathcal{X}_N$ be \textit{any} finite partition of $\mathcal{X}$. \\

Define $p_j := \mathbb{P}(X \in \mathcal{X}_j) = \int_{\mathcal{X}_j}f(x)dx$, for $j =1,...,N$. \\

Define $J : \Omega \to \{1,2,...,N\}$ by $J(\omega)=j \iff X(\omega) \in \mathcal{X}_j$. Note that since $(\mathcal{X}_j)_{j=1}^N$ is a partition, for any given $\omega \in \Omega$, $X(\omega)$ can be in one and only one set $\mathcal{X}_j$, so $J$ is well defined. Additionally, $J$ is a discrete random variable since $X$ is a random variable and its image is a finite set. Note that $\mathbb{P}(J=j)=\mathbb{P}(X \in \mathcal{X}_j) = p_j$\\

We will show that $h(X)-h(X|J)=H(J)$, where $H(J)$ stands for the entropy of $J$ given by \[
H(J)= \sum_{j=1}^N p_j \log p_j
\]In other words, $h(X)-h(X|X \in\mathcal{X}_j)=H(X \in \mathcal{X}_j)$. \\

Now, note that $h(X)-h(X|J) = I(X;J)$. By symmetry of mutual information, $I(X;J)=I(J;X)=H(J)-H(J|X)$, since $J$ is a discrete random variable. Also, since $J$ is a deterministic function of $X$, $H(J|X)=0$ and thus $h(X)-h(X|J)=I(J;X)=H(J)$.\\

Now, recalling the known fact that if $J$ is a discrete random variable over $\{1,2,...,N\}$, $H(J) \leq \log N$, with equality if, and only if, $p_j = \frac{1}{N}, \forall j =1,2,...,N$ (e.g, the discrete uniform is the only one with maximum entropy) the result follows. \\

In particular, consider any continuous prior over our set of environments in Definition A.2 (i.e, consider it to have any density function $f(x)$). Then, consider the same deterministic, stationary policies $\Pi =\{\pi: S \to A\}$. Since $\Pi$ is finite, for every environment there is at least one optimal policy, so we can construct a function $f: \mathcal{X} \to \Pi$ that takes every environment $x$ to a policy that is optimal $x$. Then, $\ker f$ is an equivalence relation and induces a partition  $\mathcal{X}_1, ..., \mathcal{X}_N$ with a finite number of sets since there is a finite number of policies. Then, since the entropy reduction by knowing a certain policy is optimal is given by $h(X)-h(X|X \in \mathcal{X}_j)$, it will always be true that $h(X)-h(X|X \in \mathcal{X}_j) \leq \log N$

\section{Extension to a General Prior} \label{ap:generalprior}
In the body of the paper, we proved Lemma \ref{lem:nlogm} for the special case where the prior over environments $p(x)$ was uniform. Here, we will prove a similar result which applies more generally to all priors which are invariant up to a re-labelling of the actions, provided that the value function is affected by the policy only through the transition matrix. This assumption is required in the proof of Theorem \ref{thm:equalvolumes} and applies to all types of reward maximization considered in the main paper. \\

In this proof, we will use the family of functions $g_{ij}:\mathcal{X}\rightarrow \mathcal{X}$ which were introduced in the proof of Theorem \ref{thm:equalvolumes} in the main paper and are developed in Appendix  \ref{ap:equalvolumes}. To recap, $g_{ij}(x)$ maps an environment $x$ to an environment $x'$ for which the actions are re-labelled so that transition probababilities in each state are the same for policy $\pi_i$ acting in environment $x$ and policy $\pi_j$ acting in environment $g_{ij}(x)$. In terms of transition matrices, we have $M_{\pi_i}(x) = M_{\pi_j}(g_{ij}(x))$. \\

If a prior is invariant up to a re-labelling of actions, we have that that $p(x) = p(g_{ij}(x)) \, \forall i,j,x$. Note that the uniform prior $p(x)=\frac{1}{V}\, \forall x$ is one example of a prior that satisfies this condition. Now, let us calculate the $h(X)$ the differential entropy of a prior of this kind.
\be
    h(X) = -\int_{\mathcal{X}} p(x) \log p(x) dx \, .
\ee
By the assumptions of Lemma \ref{lem:nlogm}, we have that $(\mathcal{X}_i)_{i \in 1 ...m^n}$ forms a partition for all environments except a measure zero subset. Therefore, we have:
\be
    h(X) = -\sum_{i}^{m^n}\int_{\mathcal{X}_i} p(x) \log p(x) dx \, .
\ee
Now, for any value function which is only influenced by the policy through its effect on the transition matrix $M_\pi(x)$, we have that $\mathcal{X}_j = g_{ij}(\mathcal{X}_i)$ and the volumes of $\mathcal{X}_i$ and $\mathcal{X}_j$ are equal for all $i,j$. For a proof of this, see Theorem \ref{thm:equalvolumes} and its proof in Appendix \ref{ap:equalvolumes}. Using these facts (along with the fact that $p(x)=p(g_{ij}(x))$) we have:
\be
    \int_{x \in \mathcal{X}_i} p(x)\log p(x) dx = \int_{g_{ij}(x) \in \mathcal{X}_j} p(g_{ij}(x))\log p(g_{ij}(x)) dx \, .
\ee
Now we perform a change of variables $y=g_{ij}(x)$ and use the fact found in Appendix \ref{ap:Jacobean} that the determinant of the Jacobean of $g_{ij}$ is one. This yields
\be
    \int_{x \in \mathcal{X}_i} p(x)\log p(x) dx = \int_{x \in \mathcal{X}_j} p(x)\log p(x) dx \, .
\ee
This allows us to write

\begin{align}
    h(X) & = -\sum_{i}^{m^n}\int_{\mathcal{X}_i} p(x) \log p(x) dx \\
    & =-m^n \int_{\mathcal{X}_i} p(x) \log p(x) dx \, .
\end{align}
Now, we will find how this changes when we update upon discovering the optimal policy. Upon learning the optimal policy, the updated distribution over environments is
\be
    p(x|\Pi = \pi_i) = 
    \begin{cases}
        m^n \, p(x) \text{ if } x \in \mathcal{X}_i \\
        0 \text{ otherwise} \, .
    \end{cases}
\ee
where we have used Bayes rule with $p(\Pi=\pi_i|x) = \mathbb{I}_{x \in\mathcal{X}_i}$ and $p(\Pi=\pi_i) = \frac{1}{m^n}$.
The differential entropy of this distribution is
\begin{align}
    h(x|\Pi=\pi_i) & = - \int_\mathcal{X} p(x|\Pi = \pi_i) \log p(x|\Pi = \pi_i) dx \\
    & = - \int_{\mathcal{X}_i} m^n p(x) \log( m^n p(x)) dx \\ 
    & = - m^n \int_{\mathcal{X}_i} p(x) \log p(x) dx - m^n \int_{\mathcal{X}_i}  p(x) \log m^n dx \\
    & = h(X) - m^n \int_{\mathcal{X}_i} p(x)\log m^n dx  \, .
\end{align}
Finally, we note that by normalisation we must have $\int_\mathcal{X} p(x)dx =1$. By a similar argument as before, we can write:
\be
    1=\int_\mathcal{X} p(x)dx = m^n \int_{\mathcal{X}_i} p(x) dx
\ee
By symmetry, this means we must have $\int_{\mathcal{X}_i} p(x) dx = \frac{1}{m^n},  \forall i$. As a result, we can write
\be
    h(X|\Pi=\pi_i) = h(X) - \log m^n \, .
\ee
This expression can then be re-arranged to and summed over $i$ to obtain the mutual information:
\be
    I(X:\Pi) = h(X) - \sum_i h(X|\Pi=\pi_i)p(\Pi = \pi_i) = n \log m\, .
\ee
Therefore, for any prior invariant under a re-labelling of actions and any value function which is affected by the policy only through the transition matrix. This includes all reward maximization types considered in the main paper.
\section{Proof of Theorem \ref{thm:equalvolumes}} \label{ap:equalvolumes}
First, we will be precise about what we mean when we say that a value function ``is only influenced by the policy through its effect on the transition matrix $M_\pi(x)$''. Though it is fairly intuitive, we will now make it mathematically precise.\\

Let $\bar{M}:\Pi \times \mathcal{X} \rightarrow \mathbb{R}^{n\times n}$ be the function which takes a policy and an environment and outputs the transition matrix. Let $\bar{V}:\Pi \times \mathcal{X} \rightarrow \mathbb{R}$ be a value function which takes in a policy an environment and outputs a total reward. We say that $\bar{V}$ ``is only influenced by the policy through its effect on the transition matrix $M_\pi(x)$'' if it is possible to find a function $k:\mathbb{R}^{n\times n} \rightarrow \mathbb{R} $ such that $k\circ\bar{M} = \bar{V}$. In this sense, the only way in which the policy can change the total reward is through its effect on the transition matrix. The class of value functions which can be expressed this way is very broad and (as we show in the main paper) applies to all of the main forms of reward maximisation. However, it is not so broad as to include all conceivable goals. For example, goals involving achieving a precise sequence of actions and states (such as the goals considered in \cite{richens2025general})are not expressible in this way. \\

We have defined the family of functions $g_{ij}: \mathcal{X} \to \mathcal{X}, \forall i,j \in \{1,2,...,n\}$. In this appendix, we will use these functions to show that, they are invertible, preserve optimality, and preserve volume, i.e,  \[g_{ij}(\mathcal{X}_i) = \mathcal{X}_j \, ,\] \[\int_{x \in \mathcal{X}_i} dx = \int_{x \in \mathcal{X}_j} dx \, ,\] and \[M_{\pi_i}(x) = M_{\pi_j}(g_{ij}(x)) \, .\]

$\forall i =1...n, j=1...n$. \\

To fix notation, we'll index $x \in \mathcal{X} \subset \mathbb{R}^{mn^2}$ as $x = (x_{s,a,s'})_{s,s'\in S, a \in A}$, noting that $x_{s,a,s'} = P(s'|s,a)$. Note that determining an environment is equivalent as determining $mn^2$ probabilities $x_{s,a,s'}$ such that $\sum_{s' \in S} x_{s,a,s'}=1, \forall (s,a) \in S \times A $.\\

\subsection{Characterizing $g_{ij}$}



In the main body of the paper, we described the functions $g_{ij}$ as mapping from an environment $x$ to another environment $x'$, where the actions of policy $\pi_j$ have the same effect as the actions of policy $\pi_i$ do in $x$. Formally, we can write this as: \be
g_{ij}(x)_{s, \pi_j(s),s'} := x_{s,\pi_i(s), s'} \, ,
\ee
\be
g_{ij}(x)_{s,\pi_i(s),s'} := x_{s, \pi_j(s), s'} \, .
\ee
For transition probabilities corresponding to actions not taken by $\pi_i$ or $\pi_j$, the environment is left unchanged so we have for $a \notin \pi_i(S)\cup \pi_j(S)$, 
\be
g_{ij}(x)_{s,a,s'} := x_{s,a,s'} \, .
\ee
Now note that since $g_{ij}$ function which simply swaps some pairs of entries in a $mn^2$-dimensional vector it is its own inverse (since swapping a pair of entries twice results in the original vector). \\

We can see that $g_{ij} \circ g_{ij} = id$ more explicitly by considering the transformation on all individual components of $x$. By construction of $g_{ij}$ (temporarily dropping the $ij$ subscript from $g_{ij}$ to reduce notational clutter) we have,
\begin{align}
    g(g(x))_{s,\pi_i(s), s'} &=g(x)_{s,\pi_j(s),s'}\\
    &=x_{s,\pi_i(s),s'} \, ,
\end{align}
\begin{align}
    g(g(x))_{s,\pi_j(s), s'} &=g(x)_{s,\pi_i(s),s'}\\
    &=x_{s,\pi_j(s),s'} \, .
\end{align}
Also, for $a \notin \pi_i(S)\cup\pi_j(S)$,
\begin{align}
    g(g(x))_{s,a, s'} &=g(x)_{s,a,s'}\\
    &=x_{s,a,s'} \, .
\end{align}
Thus, $g(g(x))=x, \forall x \in \mathcal{X}$.
\subsection{The effect on transition matrices}
For this theorem, we are concerned with the effects of the policy on the result transition matrix for that policy acting in a particular environment. Due to the construction of $g_{ij}$, the transition matrix induced by policy $\pi_i$ in environment $x$ is the same as the transition matrix induced by policy $\pi_j$ in environment $g_{ij}(x)$. Using our notation $M_\pi(x)$ to indicate the transition matrix induced by policy $\pi$ in environment $x$, we can write:
\be
    M_{\pi_i}(x) = M_{\pi_j}(g_{ij}(x)) \, .
\ee
This can be seen by considering these matrices elementwise, using $k,l$ to denote matrix indexes:
\begin{align}
    [M_{\pi_i}(x)]_{kl} &= P(s_l|s_k,\pi_i(s_k))\\
    &=x_{s_k,\pi_i(s_k),s_l}\\
    &=g(x)_{s_k,\pi_j(s_k),s_l}\\
    &=[M_{\pi_j}(g(x))]_{kl} \, .
\end{align}
Since $g_{ij}$ is its own inverse, by symmetry, we also have 
\be
    M_{\pi_i}(g_{ij}(x)) = M_{\pi_j}(x) \, .
\ee

\subsection{Optimality}
In this section, we prove that if $\pi_i$ is optimal in an environment $x$, then $\pi_j$ is optimal in environment $g_{ij}(x)$. This is proved for \emph{any} value function $V_{\pi}(x)$ where the contribution of the policy is determined only by the transition matrix $M_{\pi}(x)$ induced by policy $\pi$. As it turns out, this covers a very broad class of goals, including all forms of reward maximisation (time-averaged, time-discounted, and finite time) considered in this main paper. \\

Now, we will show that, for any value function in this class, if $\pi_i$ is optimal in an environment $x$, then $\pi_j$ is optimal in environment $g_{ij}(x)$ and achieves the same performance. This can be seen intuitively by noting the fact that, since $g_{ij}$ swaps the effect of taking some actions in some states, while leaving all other transition probabilities unchanged, each policy in $x$ will have a corresponding policy with the same performance in $g_{ij}(x)$. We will now make this claim formal. \\

Let $\tilde{\Pi} =\{\pi: S \to A\} =\{\pi_i|i\in\mathcal{I}\}$ be the space of all deterministic, stationary policies. Then, define $\phi_{ij}:\tilde{\Pi}\rightarrow \tilde{\Pi}$ be a function which maps policies to policies. We will use $\phi_{ij}[\rho]:S\rightarrow A$ to denote the policy resulting from applying $\phi_{ij}$ to a policy $\rho$. We will use $\phi_{ij}[\rho](s)$ to indicate the action resulting from applying the policy $\phi_{ij}[\rho]$ to a particular state $s$. We define $\phi_{ij}$ in terms of its effect on a generic policy:

\be
\phi_{ij}[\rho](s) = \begin{cases}
    \pi_j(s), \text{ if } \rho(s)=\pi_i(s) \\
    \pi_i(s), \text{ if } \rho(s)=\pi_j(s) \\
    \rho(s), \text{ otherwise} \, .
\end{cases}
\ee
In other words, for states where $\rho$ takes the same action as either $\pi_i$ or $\pi_j$, we have that $\phi_{ij}$ switches which action is taken. The behaviour of the policy is left unchanged in states where $\rho$ behaves differently to $\pi_i$ or $\pi_j$. Note that $\phi_{ij}$ is its self-inverse. Since it only swaps the actions taken, we have $\phi_{ij} \circ\phi_{ij} =id$. \\

We will now show that, for any reward function which depends only on the transition matrix $M_\pi(x)$ the performance of policy $\rho$ in environment $x$ is equal to the performance of policy $\phi_{ij}[\rho]$ in environment $g_{ij}(x)$ ie. $V_{\phi_{ij}[\rho]}(x) = V_{\rho}(g_{ij}(x))$. Going forward, we will use $\tilde{\rho} = \phi_{ij}[\rho]$ to indicate a policy $\rho$ transformed by $\phi_{ij}$. \\

Since we are considering value functions which only depend on the transition probabilities $x_{s,\pi(s),s'}$, it suffices to show $x_{s, \tilde{\rho}(s), s'} = g_{ij}(x)_{s,\rho(s),s'}$. We will now show this is true for all $s,s'$.  \\

First, we consider cases where $\rho(s) = \pi_i(s)$. Then, by definition of $g_{ij}$ and $\phi_{ij}$, if $\rho(s) = \pi_i(s)$ we have:

\begin{align}
    x_{s, \tilde{\rho}(s), s'} &= x_{s, \pi_j(s), s'} \\
    &=g_{ij}(x)_{s, \pi_i(s),s'} \\
    &=g_{ij}(x)_{s, \rho(s), s'} \, .
\end{align}

Analogously, if $\rho(s) = \pi_j(s)$, it follows $x_{s,\tilde{\rho}(s), s'} = g_{ij}(x)_{s, \rho(s), s'}$. Finally, for states $s$ where $\rho(s) \neq \pi_i(s)$ and $\rho(s) \neq \pi_j(s)$, then we have that $\tilde{\rho}(s)=\rho(s)$. Furthermore, the relevant environmental transition probabilities will be left unchanged so we have $g_{ij}(x)_{s,\rho(s),s'}=x_{s,\rho(s),s'}$. Meaning that for these states, we can write:
\begin{align}
    x_{s, \tilde{\rho}(s), s'} &= x_{s, \rho(s), s'} \\
    &= g_{ij}(x)_{s, \rho(s), s'} \, .
\end{align}
We have therefore shown that for all $s,s'$, the elements of the transition matrix for a policy $\rho$ in an environment $x$ are the same as the elements of the transition matrix for a policy $\tilde{\rho} = \phi_{ij}(\rho)$ in an environment $g_{ij}(x)$. \\

Note that $\phi_{ij}(\pi_i)=\pi_j$, by definition of $\phi_{ij}$. Assume that $\pi_i$ is optimal in environment $x$. Then, it holds that $V_{\pi_i}(x) \geq V_{\pi}(x), \forall \pi \in \tilde{\Pi}$. Now, for any $\rho \in \tilde{\Pi}$, since $\phi_{ij}$ is invertible, there is $\rho' \in \tilde{\Pi}$ such that $\phi_{ij}(\rho')=\rho$. Then, \begin{align}
    V_{\pi_j}(g(x)) &=  V_{\phi_{ij}(\pi_j)}(x)\\
    &=V_{\pi_i}(x) \, ,
\end{align}
ie. the performance of $\pi_j$ in environment $g_{ij}(x)$ is equal to the performance of $\pi_i$ in environment $x$. Furthermore:
\begin{align}
    V_{\pi_i}(x) &\geq  V_{\phi_{ij}(\rho')}(x) \\
    &= V_{\rho'}(g_{ij}(x)) \, .
\end{align}
Then, $V_{\pi_j}(g_{ij}(x)) \geq V_{\rho}(g_{ij}(x)), \forall \rho \in \tilde{\Pi}$. Thus, if $\pi_i$ is optimal in environment $x$, we have that $\pi_j$ is optimal and with the same performance in environment $g_{ij}(x)$.

\subsection{Each $g_{ij}$ preserves volume} \label{ap:Jacobean}
We'll use the theorem found in \cite{calculus1965spivak} that if $U,V \subseteq \mathbb{R}^N$ open sets and $g : U \to V$ continuously differentiable bijection with differentiable inverse, then for any $A \subseteq U$ measurable, \be
vol(g(A)) = \int_{A}|\det J_g(x)|dx \, .
\ee

First note that $g_{ij} : \mathcal{X} \to \mathcal{X}$ is a linear function. Thus, there is a $mn^2 \times mn^2$ matrix $P$ such that $g_{ij}(x)=Px$, so $g_{ij}$ is differentiable and $J_g(x)=P$. Note now that since $g \circ g = id$, then $P^T \cdot P=Id$, where $Id$ stands for the $mn^2 \times mn^2$ identity matrix, and hence by $\det(P^TP)=1 \implies \det(P)^2 =1 \implies |\det(P)|=|\det(J_g)|=1$, since $\det(P^T)=\det(P)$ and $\det(AB)=\det(A)\cdot \det(B)$, for any $N \times N$ matrixes $A,B$. \\

We can restrict $g$ to the interior of $\mathcal{X}_i$ and then apply the theorem, since the border of $\mathcal{X}_i$ has measure zero and therefore doesn't influence the integral. Then, since there is a volume-preserving transformation between all pairs of sets $\mathcal{X}_i$ and $\mathcal{X}_j$ we can conclude that each of these sets has the same volume.

\section{Proof of Lemma \ref{lem:averagerealanalytic}} \label{ap:averageanalytic}

We will  first apply standard results from the study of Markov processes to show that, for any policy $\pi_i$ acting in any environment $x\in \text{int}(\mathcal{X})$ and with any initial probability distribution over states $\vec{v}_0$, the probability distribution over states will converge to a steady state. Following this, we will show that the steady state is a real analytic function of $x$ and that, due to the Ces\`{a}ro mean theorem, this means that $V_{\pi_i}$ and $f_{ij}$ are also real analytic functions of $x$. \\

If we fix a deterministic, Markovian policy $\pi_i$ and let it run in an environment $x$, then we can treat the system as a stationary Markov process. Restricting to $x\in \text{int}(\mathcal{X})$ means that none of the transition probabilities in this Markov process are zero. This means that the Markov process is both \emph{irreducible} and \emph{aperiodic}. \emph{Irreducible} means that, for any two states $s_a$ and $s_b$, when starting from $s_a$, there exists some number of timeteps after which there is a nonzero probability of the system eventually transitioning to $s_b$. \emph{Aperiodicity} is enforced when, for all states $s_a$, there is some probability of remaining in that state at the subsequent timestep. Our condition $x\in \text{int}(\mathcal{X})$ is actually stronger than just enforcing aperiodicty and irreducibility, as it means that none of the transition probabilities are exactly zero, so for every state-action pair, there is a nonzero probability of transitioning to any state. But since restricting to $\text{int}(\mathcal{X})$ only involves excluding a measure zero subset of environments we can do so without any additional loss of generality. It is a  standard result that a Markov process which is irreducible and aperiodic will eventually converge to a stationary distribution over states, regardless of the starting state.
\begin{theorem} \label{Thm:convergence}
    Convergence theorem. If a Markov chain (a Markov process with discrete states and discrete timesteps) is characterised by a transition matrix $M$ which is irreducible and aperiodic then:
    \begin{enumerate}
        \item There exists a unique stationary distribution $\vec{\mu}$ such that $M \vec{\mu} = \vec{\mu}$.
        \item Any initial state $\vec{v}_0$ will converge to $\vec{\mu}$ with repeated applications of $M$.
    \end{enumerate}

\end{theorem}
\begin{proof}
    Proofs of \emph{1.} and \emph{2.} can be found in most textbooks on Markov Chains such as \cite{norris1998markov, levin2017markov}.
\end{proof}

We will use $\mu_{\pi}(x)$ to denote the stationary distribution resulting from a policy $\pi$ acting in an environment $x$. The fact that the distribution over states converges to a stationary distribution greatly simplifies the analysis of $V_\pi(x)$ due to the following standard result.
\begin{lemma}
    Ces\`{a}ro mean theorem (also known as Cauchy's limit theorem). Let $(a_n)_{n \in \mathbb{N}}$ a convergent sequence of real numbers such that $\lim_{n\rightarrow\infty} a_n = L$. Then 
    \be
        \lim_{N\rightarrow \infty} \frac{1}{N}\sum_{i=1}^N{a_n} = L \, .
    \ee
\end{lemma}
\begin{proof}
    The proof is standard and can be found in many textbooks such as \cite{knopp2012infinite}, under the name `Cauchy's limit theorem'.
\end{proof}

Applying this theorem to our value function gives:
\be
    V_{\pi}(x) = \lim_{T\rightarrow \infty}\frac{1}{T}\sum_{t=1}^T R(M_\pi(x)^t\vec{v}_0) = \lim_{t \rightarrow \infty} R(M_\pi(x)^t\vec{v}_0) = R(\vec{\mu}_\pi(x)) \, .
\ee
The average value function simplifies down to expected reward of the stationary distribution. We will complete the proof by showing that $\mu_\pi(x)$ and therefore $V_\pi(x)$ is a real analytic function of $x$. \\

\begin{lemma} \label{lem:Vpianalytic}
    $V_\pi(x)$ is a real analytic function of $x$.
\end{lemma}
\begin{proof}
    All we need to do is prove that the steady state vector is a real analytic function of the elements of $M_\pi(x)$ and hence a real analytic function of $x$. We will use $\vec{\mu} = \vec{\mu}_\pi(x)$ to denote this steady state which is the solution to the equation $(M_\pi(x) - 1) \vec{\mu} = 0$. However, since $M_\pi(x)$ has an eigenvalue equal to 1 (by virtue of being aperiodic and irreducible) we have that $\text{det}(M_\pi(x) -1) =0$, meaning that we cannot apply Cramer's rule to express $\vec{\mu}$ in terms of $M_\pi(x)$. \\

We will use the fact the $\vec{\mu}$ must be a normalised probability vector to solve this problem. By the Perron-Frobenius theorem, we know that $M_\pi(x)$ has eigenvalues $|\lambda|\leq1$ with exactly one eigenvalue equal to 1. As a result, the matrix $M_\pi(x) -1$ has a 1-dimensional nullspace and therefore by the rank-nullity theorem has rank $(n-1)$. This means that one of the rows of $M_\pi(x)-1$ is not linearly independent of the others. Without loss of generality, we will assume that this is the last row. \\

Let us use $B$ to denote the matrix $M_\pi(x)-1$ with its last row replaced with all ones. First, we will show that the all-ones row is linearly independent of the other rows, meaning that $B$ is full rank and thus a nonzero determinant. We then apply Cramer's rule to express $\vec{\mu}$ in terms of $B$ (and hence, in terms of $M_\pi(x)$). \\

Assume that the all ones row is not linearly independent of the other rows. This would mean that there exists a vector $\vec{u} \neq\vec{0}$ such that $B \vec{u} =0$. Let us abbreviate $M_\pi(x)$ as $M$. Since $B$ and $M-1$ are the same for the first $n-1$ rows, this means that we also have $[(M-1)u]_i=0$ for $i = 1...n-1$. Since $M$ is a stochastic matrix, each of the columns of $M$ sum to one. This means that $\vec{1}^\intercal (M-1) = \vec{0}$ where $\vec{0}$ and $\vec{1}$ are the $n$-dimensional column vectors of all zeros and all ones respectively. As a result must have $\vec{1}^\intercal (M-1) \vec{v} = 0$ for any vector $\vec{v}$. Therefore:
\be
    \vec{1}^\intercal (M-1) \vec{u} = \sum_i [(M-1)\vec{u}]_i = 0 \, .
\ee
Since we have $[(M-1)u]_i=0$ for $i = 1...n-1$, this means that $[(M-1)\vec{u}]_i$ must also be zero meaning that $(M-1)\vec{u} = \vec{0}$. As we noted before, the nullspace of $M-1$ is 1-dimensional. We therefore require that $\vec{u}$ must be in the nullspace along with the steady state $\vec{\mu}$. Therefore it must be possible to express $\vec{u}=\alpha   \vec{\mu}$ using some scalar $\alpha$. \\

Now, notice that the last row of $B\vec{u}$ is simply the sum of the elements of $\vec{u}$, since the last row of $B$ is all ones. Simultaneously, we have assumed that $B\vec{u}=\vec{0}$ so we have $[B\vec{u}]_n = \sum_i^{n} u_i =0$. But, since we also have $\vec{u}=\alpha \vec{\mu}$, and $\sum_i \mu_i =1$, the only way for this to be true is if $\alpha=0$. But this contradicts our supposition that $\vec{u}\neq \vec{0}$, meaning that no such vector $\vec{u}$ exists. As a result, we can conclude that all rows of $B$ are linearly independent. Thus $B$ is full rank and invertible. \\

Finally, notice that, due to the normalisation of $\mu$, we have $B \vec{\mu} =  \vec{w}$ where $\vec{w} = (0,0,..1)^\intercal$. As a result, we can express the elements of $\mu$ using Cramer's rule:
\be
    \mu_i = \frac{\det{B_i}}{\det{B}} \, ,
\ee 
where $B_i$ is the matrix resulting from replacing the $i^{th}$ column of $B$ with $\vec{w}$. As we have proved that $B$ is invertible, we have that $\det{B}\neq0$ meaning that these expressions are well-defined. Since each of these expressions for $\mu_i$ are simply polynomials in the elements of the matrix $M-1$, the elements of $\mu_i$ are real analytic functions of the elements of $M_\pi(x)$ and hence real analytic functions of $x$.

\end{proof}
\section{Proof of Lemma \ref{lem:notzero}} \label{ap:notzero}
Since we have that $\pi_i$ and $\pi_j$ must disagree on what action to take in at least one state, we pick one state where they take different actions and call it $s_a$. Furthermore, since we have a non-constant reward function, we must have one state with the highest reward. Call this state $s_b$. It may be the case that $s_a$ and $s_b$ are the same states, or they may be different. Now consider the following environment $x'$. In this environment, in all states except $s_a$, all actions result in the next state being a random state from a uniform distribution over states. This means that the transition matrix elements $[M_{\pi_i}(s)]_{kl} = [M_{\pi_j}(s)]_{kl}=\frac{1}{n}$ for all $k,l$ except $l=a$. \\

In state $s_a$, the action taken by $\pi_i$ has a high probability of $1-\epsilon$ of transitioning the system to state $s_b$ and a $\frac{\epsilon}{n-1}$ probability of transitioning to each of the other states. On the other hand, the action taken by state $\pi_j$ in state $s_a$ leads to each of the other states with equal probability $\frac{1}{n}$. This means that every element of $M_{\pi_j}(x)$ is $\frac{1}{n}$. \\

For $\pi_i$, on the other hand, every element of $M_{\pi_i}(x)$ is $\frac{1}{n}$ \emph{except for} the $a^{th}$ column. In the $a^{th}$ column of $M_{\pi_i}(x)$, every element is $\frac{\epsilon}{n-1}$ except for the element in the $a^{th}$ row, which is $1-\epsilon$. By computing the eigenvectors of $M_{\pi_i}(x')$ and $M_{\pi_j}(x')$, it is then straightforward to verify that the steady state distribution for $\pi_j$ is the unform distribution over all $n$ states. Whereas the steady state distribution for $\pi_i$ has probability mass $\frac{1}{n \epsilon +1}$ on state $s_b$ and probability mass $\frac{1}{(n \epsilon +1)(n-1)}$ on each of the other states. Since $s_b$ has a higher reward than all other states, we can conclude that policy $\pi_i$ achieves a higher performance than $\pi_j$ in environment $x'$. This means that $f_{ij}(x')>0$, which proves that $f_{ij}(x)$ is not identically zero. 

\bibliographystyle{unsrt}
\bibliography{ModelBits.bib}
\end{document}